\documentclass[a4paper, 10pt, conference]{ieeeconf}      

\IEEEoverridecommandlockouts                              

\let\labelindent\relax
\usepackage{fix-cm}
\usepackage{etex}

\usepackage{dblfloatfix}

\usepackage{nag}

\makeatletter
\@ifpackageloaded{xcolor}{}{%
\usepackage[table,x11names,dvipsnames,svgnames]{xcolor}%
}
\makeatother

\usepackage{colortbl}

\usepackage{graphicx}
\usepackage{wrapfig}

\usepackage{cite}

\usepackage{microtype}

\usepackage{array}
\usepackage{multirow}
\usepackage{booktabs}
\usepackage{makecell} 

\newcommand{\myparagraph}[1]{\textbf{\emph{#1}}.}

\ifcsname labelindent\endcsname
\let\labelindent\relax
\fi
\usepackage[inline]{enumitem}

\usepackage{subfig}

\newenvironment{lenumerate}[2][]
{\begin{enumerate}[label=(#2\arabic*),leftmargin=0.2in,itemindent=0.15in,#1]}
{\end{enumerate}}

\setlist*[enumerate,1]{label={\itshape\arabic*)}}

\makeatletter
\newcommand{\paragraphswithstop}{%
\let\copyparagraph\paragraph%
\renewcommand\paragraph[1]{\copyparagraph{##1.}}%
}
\makeatother

\usepackage[framemethod=tikz]{mdframed}

\input{math}
\newcommand{\real}[1]{\mathbb{R}^{#1}{}}

\DeclarePairedDelimiter{\norm}{\lVert}{\rVert}

\newcommand{\vct}[1]{\mathbf{#1}}

\DeclareMathOperator*{\argmin}{\arg\!\min}
\DeclareMathOperator*{\argmax}{\arg\!\max}

\input{utilities}

\providecommand{\mI}{\vct{I}}

\providecommand{\mP}{\vct{P}}

\providecommand{\mR}{\vct{R}}

\providecommand{\mT}{\vct{T}}

\providecommand{\cC}{\mathcal{C}}

\providecommand{\cE}{\mathcal{E}}

\providecommand{\cG}{\mathcal{G}}

\providecommand{\cL}{\mathcal{L}}

\providecommand{\cN}{\mathcal{N}}
\providecommand{\cO}{\mathcal{O}}

\providecommand{\cS}{\mathcal{S}}
\providecommand{\cT}{\mathcal{T}}

\providecommand{\cV}{\mathcal{V}}

\usepackage{units}

\newcommand{\newcolorlabel}[2]{%
  \expandafter\newcommand\csname #1\endcsname[1]{%
    \colorbox{#2}{\color{white}\textsf{\textbf{##1}}}}%
}

\newcommand{\newcommenter}[2]{%
  \expandafter\newcommand\csname #1\endcsname[1]{%
    \fcolorbox{#2}{#2}{\color{white}\textsf{\textbf{#1}}}
    {\color{#2}##1}}%
  \expandafter\newcommand\csname at#1\endcsname{%
    \fcolorbox{#2}{#2}{\color{white}\textsf{\textbf{@#1}}}
    {\color{#2}}}%
  \expandafter\newcommand\csname #1hl\endcsname[2]{%
    \colorbox{#2}{\color{white}\textsf{\textbf{#1}}}\sethlcolor{Azure2}\hl{##2}~%
    \expandafter\ifx\csname commentarrow\endcsname\relax$\leftarrow$\else \commentarrow[#2]\fi~%
    {\color{#2}##1}}%
  \expandafter\newcommand\csname #1st\endcsname[2]{%
    \colorbox{#2}{\color{white}\textsf{\textbf{#1}}}\sout{##2}~%
    \expandafter\ifx\csname commentarrow\endcsname\relax$\leftarrow$\else \commentarrow[#2]\fi~%
    {\color{#2}##1}}%
}
\newcommenter{TODO}{DodgerBlue1}
\newcommenter{rtron}{Green3}

\usepackage{comment}

\usepackage{pdfcomment}

\usepackage{soul}

\usepackage[normalem]{ulem}

\usepackage{csquotes}

\newcommand{\trp}{\mathsf{T}}
\newcommand{\SO}{\mathrm{SO}}
\newcommand{\SE}{\mathrm{SE}}
\newcommand{\SIM}{\mathrm{Sim}}
\newcommand{\So}{\mathfrak{so}}

\usepackage{bm}
\usepackage{mathrsfs}
\usepackage{pgfplots} 

\usetikzlibrary{plotmarks}
\usetikzlibrary{positioning}
\usetikzlibrary{arrows, automata}
\usetikzlibrary{decorations.markings}
\pgfplotsset{plot coordinates/math parser=false} 
\newlength\figureheight 
\newlength\figurewidth

\tikzset{
  treenode/.style = {align=center, inner sep=0pt, text centered,
    font=\sffamily},
  arn_n/.style = {treenode, circle, white, font=\sffamily\bfseries, draw=black,
    fill=black, text width=1.5em},
  arn_r/.style = {treenode, circle, red, draw=red, 
    text width=1.5em, very thick},
  arn_x/.style = {treenode, rectangle, draw=black,
    minimum width=0.5em, minimum height=0.5em}
}

\usetikzlibrary{shapes.misc}

\tikzset{cross/.style={cross out, draw=black, minimum size=2*(#1-\pgflinewidth), inner sep=0pt, outer sep=0pt},
cross/.default={1pt}}

\title{\LARGE \bf Statistical Outlier Identification in Multi-robot Visual SLAM\\using Expectation Maximization}

\author{Arman Karimian, Ziqi Yang, Roberto Tron
\thanks{The authors are with the Department of Mechanical Engineering of Boston University, 
Boston, MA. E-mail: {\tt\small \{armandok,zy259,tron\}@bu.edu}.}%
}

\begin{document}

\maketitle
\thispagestyle{empty}
\pagestyle{empty}

\begin{abstract} 



This paper introduces a novel and distributed method for detecting inter-map loop closure outliers in simultaneous localization and mapping (SLAM). The proposed algorithm does not rely on a good initialization and can handle more than two maps at a time. In multi-robot SLAM applications, maps made by different agents have nonidentical spatial frames of reference which makes initialization very difficult in the presence of outliers. This paper presents a probabilistic approach for detecting incorrect orientation measurements prior to pose graph optimization by checking the geometric consistency of rotation measurements. Expectation Maximization is used to fine-tune the model parameters. As ancillary contributions, a new approximate discrete inference procedure is presented which uses evidence on loops in a graph and is based on optimization (Alternate Direction Method of Multipliers). This method yields superior results compared to Belief Propagation and has convergence guarantees. Simulation and experimental results are presented that evaluate the performance of the outlier detection method and the inference algorithm on synthetic and real-world data.

\end{abstract}

\section{Introduction}
Geometric mapping from unknown sensor positions has a long history. This task, which is known as \emph{Structure from Motion}~(SfM) in the computer vision, is traditionally done using only images. The solution pipeline~\cite{hartley2003multiple} for this problem includes three steps. First, estimate relative poses between pairs of images using matched features~\cite{lowe2004distinctive,dong2015domain,bay2008speeded} and robust fitting techniques~\cite{fischler1981random,hartley2012efficient}.
Second, combine the pairwise estimates either in sequential stages~\cite{agarwal2011building, agarwal2010bundle,frahm2010building,snavely2006photo,snavely2008skeletal} or by using combining poses alone (without considering a 3-D structure) in a \emph{pose-graph} approach~\cite{carlone2015initialization}. 
The fourth and last step is to use Bundle Adjustment~(BA)~\cite{engels2006bundle,hartley2003multiple,triggs1999bundle}, which minimizes the reprojection error by considering jointly the motion and the structure.

Building a map of an unknown environment is an essential step for robot locomotion in GPS-denied environments. Simultaneous Localization and Mapping (SLAM) is the task of estimating the state of a robot while building a map of the environment at the same time, which has received a lot of attention in the past three decades due to its widespread applications \cite{cadena2016past}. Visual SLAM is a variant of the SLAM problem where only visual information obtained from a camera is used for the task \cite{taketomi2017visual}.

The state of the art approach for SLAM is based on a pose graph formulation where nodes represent robot poses at different instances and landmark positions, and edges represent relative pose measurements between node pairs. The full trajectory of the robot is estimated from all of measurements by finding a maximum a posteriori (MAP) solution \cite{olson2006fast,dellaert2006square}. This step is typically carried out by least squares error minimization and is highly sensitive to initialized values and the unavoidable presence of outlier measurements.

Edges in the pose graph can be divided into two categories: \emph{ego motion} edges which correspond to temporally close measurements; e.g.\ visual odometry measurements, and \emph{loop closure} edges which correspond to temporally distant measurements, e.g.\ when a location is revisited. Due to abundance of ego-motion measurements and high certainty in data associations made within them in visual SLAM, outliers in such measurements can be removed by a sliding window optimization approach, as suggested in \cite{strasdat2011double}, using robust M-estimators \cite{bosse2016robust,hartley2013rotation} or more advanced methods \cite{tron2016survey,carlone2015initialization}. Outliers in loop closure edges, however, pose a great challenge and are arguably the main cause of failures in the mapping procedure. Such edges are mainly caused by perceptual aliasing, i.e., different locations in the environment that appear similar or produce similar perceptual features.

Previous approaches for outlier detection usually rely on an initial trajectory guess and either try to mitigate the effect of outliers using M-estimators\cite{agarwal2013robust,olson2013inference,lee2013robust} or attempt to directly identify them\cite{sunderhauf2012switchable,sunderhauf2012towards,latif2013robust,carlone2014selecting,graham2015robust} within a single map. When there are more than one map, e.g.\ when multiple robots are performing SLAM or when multiple maps are generated by a robot, finding outliers within inter-map loop closures is more challenging since an initial guess is not available. In \cite{lajoie2019modeling} an optimization based approach is introduced which does not rely on an initial guess for a single map. In \cite{mangelson2018pairwise}, a set maximization approach is proposed for finding consistent loop closures but its limited to two maps.

\noindent\myparagraph{Paper contributions} We propose a distributed probabilistic approach for outlier detection between any number of maps. Our algorithm checks for the geometric consistency of the rotation measurements in loops within the pose-graph and decides if each loop-closure edge is an inlier or outlier, without relying on a trajectory estimate. We use a Gaussian additive noise model for rotation measurements and use the overall rotational error in cycles to infer the inlier/outlier probabilities. We use the Expectation-Maximization algorithm to fine-tune the parameters of the distribution of measurement errors and present simulation results. For the inference step required by our algorithm, we utilize Belief Propagation (BP). In addition, we will also present a novel inference algorithm based on Alternating Direction Method of Multipliers (ADMM) which has convergence guarantees.
 At the end, we present simulation results that evaluates the performance of our algorithm. We also use our algorithm to detect outliers between four real-world maps and present the results after merging.

\section{Probabilistic Model}
In this section we describe the additive Gaussian noise model used for modeling the errors on single edges and along graph cycles, as well as the graphical model used to relate the inlier versus outlier probabilities for each edge with the evidence provided by the geometric consistency of cycles.

\subsection{Modeling uncertainty in measurements}
\label{sec:edge-uncertainty-model}
We denote a pose graph by $\cG=(\cV,\cE,\cT)$ with vertices $\cV=\{1,\dots, n\}$ and edges $\cE\subseteq\cV\times\cV$, such that each edge represents a measured transformation $\tilde{\vct T}_{ij}$ between the poses of the sensor at instances $i$ and $j$, i.e., $\tilde{\vct T}_{ij}\approx \vct T_j \vct T_i^{-1}$. Each pose $\mT_i$ is represented as a member of a Lie group, e.g., $\SO(d)$, $\SE(d)$, or $\SIM(d)$ for $d\in\{2,3\}$. We denote as $\cT=\{\tilde{\vct T}_{ij}\}_{(i,j)\in\cE}$ the set of all measured relative poses.

A Lie group is a group that is also a smooth differentiable manifold; in the aforementioned matrix Lie groups, members can be represented with real valued square matrices. 

For the sake of simplicity, we limit our attention to $\SO(3)$, leaving the applications of the tools developed in this paper to other Lie groups as a venue for future work; this case leads to an easier propagation of the errors, but, as we will show in our simulations and experiments, it already provides significant benefits in the detection of outliers.


We model errors over rotations through a Gaussian distribution in local exponential coordinates, i.e., the distribution is defined in the tangent space at the mean, and mapped to the Lie group via the exponential map. More formally, we have:
\begin{equation}
\begin{aligned}
\bm\epsilon &\sim \cN(\vct 0,\vct\Sigma) \\
\tilde{\vct R} &= \exp(\hat{\bm\epsilon})\, \vct R
\end{aligned}
\end{equation}
where $\bm\epsilon\in\real{3}$ is a zero-mean Gaussian random variable with covariance matrix $\vct \Sigma\in\real{3\times 3}$, and $\hat{\bm\epsilon}\in\So(3)$ is a skew-symmetric matrix given by the \emph{hat operator}, that is,
\begin{equation}
\hat{\bm \epsilon}=
\begin{bmatrix}
0 & -\epsilon_3 & \epsilon_2 \\
\epsilon_3 & 0 & -\epsilon_1 \\
-\epsilon_2 & \epsilon_1 & 0
\end{bmatrix}.
\end{equation}

We assume that, for inlier measurements, the magnitude of the vector $\bm\epsilon$, which is the amount of rotation noise in radians, is small. This is our justification for the following lemma:

\begin{lemma}
The first order approximation of the uncertainty in composition of two uncertain rotations $\tilde{\vct R}_1\sim\cN_{\SO(3)}(\vct{R}_1,\vct{\Sigma}_1) $ and $\tilde{\vct R}_2\sim\cN_{\SO(3)}(\vct{R}_2,\vct\Sigma_2)$ is given by:
\begin{equation}
\tilde{\vct R}_2 \tilde{\vct R}_1\sim\cN_{\SO(3)}(\vct{R}_2\vct{R}_1,\vct{\Sigma}_2+\vct R_2\vct{\Sigma}_1\vct R_2^\trp)
\label{eq:uncertainty}
\end{equation}
\end{lemma}
\begin{proof}
By definition, we have:
\begin{equation}
\tilde{\vct R}_2 \tilde{\vct R}_1=\exp(\hat{\bm\epsilon}_2)\vct R_2\exp(\hat{\bm\epsilon}_1)\vct R_1
\end{equation}
by using the Adjoint of $\SO(3)$ \cite{kobayashi1963foundations}, we can transform $\bm\epsilon_1$ to the tangent space of $\vct R_2$, i.e. $\mR_2\exp(\hat{\bm\epsilon}_1) = \exp(\small(\mR_2\bm\epsilon_1\small)^\wedge)\mR_2$. By substitution we then obtain:
\begin{equation}
\begin{aligned}
\exp(\hat{\bm\epsilon}_2)\vct R_2\exp(\hat{\bm\epsilon}_1)\vct R_1 &= 
\exp(\hat{\bm\epsilon}_2)\exp(\small(\vct R_2\bm\epsilon_1\small)^\wedge)\vct R_2\vct R_1
\\&\approx
\exp\big( (\bm\epsilon_2+\vct R_2\bm\epsilon_1)^\wedge \big)\vct R_2\vct R_1 
\end{aligned}
\end{equation}
where  approximation in the last term is given by the Baker-Campbell-Hausdorff (BCH) formula \cite{wang2008nonparametric}, after ignoring terms of order of $\bm\epsilon_2\times \mR_2\epsilon_1$ and higher. 
\end{proof}
The truncation of the BCH formula is justified by our assumption that the inlier errors are relatively small.

In addition to the Gaussian noise model, we make the following assumption about $\vct\Sigma$:
\begin{assumption}
  Uncertainties in rotations are isotropic, i.e. $\vct\Sigma_i=\sigma_i^2\mI_3$ where $\mI_3$ is the identity matrix.
\label{assump-isotropic}
\end{assumption}

As a consequence of Assumption \ref{assump-isotropic}, the distribution of uncertainties in exponential coordinates is spherical, and, using \eqref{eq:uncertainty}, the distribution of the composition of a subset $\cS\subset \cV$ of noisy rotations is given by:
\begin{equation}\label{eq:prod-distribution}
\prod_{i\in\cS} \tilde{\vct R}_i \sim \cN_{\SO(3)}\Big(\,\textstyle\prod_{i\in\cS} \vct R_i \: ,\, (\textstyle\sum_{i\in\cS} \sigma_i^2)\vct I_3\Big).
\end{equation}

If all $\sigma_i$'s are equal, the resultant covariance matrix is given by $m\sigma^2\mI_3$, where $m=|\cS|$. Since the expected length of a zero-mean spherical Gaussian random variable $\bm \varepsilon\sim\cN(\vct 0,\varsigma^2\mI_d)$ is tightly bounded as $\frac{d}{\sqrt{d+1}}\varsigma\leq\mathbb{E}(\|\bm \varepsilon\|)\leq\sqrt{d}\varsigma$ \cite[Definition 3.1]{chandrasekaran2012convex}, for small enough $m$ and $\sigma$ the expected value of noise is proportional to $\sqrt{m}$. In fact, numerical experiments done in \cite[Figure~3]{enqvist2011non} validates this approximation.

We model the distribution for each measurement $R_e$ along an edge $e\in\cE$ with a Gaussian mixture model with two modes, one for being inliers and the other for outliers. We use the Bernoulli indicator variable $x_e\in\{0,1\}$ to determine whether $e$ is an inlier ($x_e=0$) or an outlier ($x_e=1$), with (user-defined) prior probabilities $p(x_e=0)=\pi_e$ and $p(x_e=1)=\bar\pi_e=1-\pi_e$ such that $\pi_e+\bar\pi_e=1$. Following on the assumption \ref{assump-isotropic}, we assume all inlier edges to have uncertainty $\sigma^2\vct I_3$ and all outlier edges to have uncertainty $\bar\sigma^2\vct I_3$, where $\bar\sigma \gg \sigma$; note that a sufficiently large value of $\bar\sigma$ in practice leads to an approximation of the uniform distribution.

\subsection{Graphical model for evidence along cycles}
A simple cycle is a closed chain of edges where each edge appears only once. Every simple cycle $c$ in a pose-graph corresponds to an ordered set of rotation measurements along the edges of the cycle, and the composition of these rotations should, ideally, be close to the identity. 

The intuition behind this idea is that, by transforming a reference frame following the rotations over a cycle, we should ideally obtain return the frame to its initial pose. More formally, we denote this overall rotation by $\tilde{\mR_c}$, defined as:
\begin{equation}
\tilde{\mR}_c =\prod_{e\in c} \tilde{\mR}_e,
\end{equation}
where the order of multiplication is based on a directed walk over $c$. We den define $z_c$ to be the geodesic distance of $\tilde\mR_c$ from the identity, 

\begin{equation}
z_c =\frac{1}{\sqrt{2}}\|\log(\tilde{\vct R}_c)\|_F =\arccos\Big(\frac{\mathrm{tr}(\tilde{\vct R}_c)-1}{2}\Big),
\end{equation}

where $\norm{\cdot}_F$ is the Frobenius norm. We use \eqref{eq:prod-distribution} to model the distribution of $\tilde\mR_c$, giving a probabilistic model for $z_c$. Note that the variance of $\tilde\mR_c$ mainly depends on the length of the cycle and the number of outliers in the cycle.


Similarly to previous work that aims to use geometric relations in cycles in Structure from Motion \cite{zach2010disambiguating,enqvist2011non}, we propose the Bayesian network depicted in Figure \ref{Fig-Bayesian} as our generative graphical model. In this model, errors in the cycles serve as evidence for inferring the hidden inlier/outlier state random variables $x_e$ for each edge $e\in\cE$. In this network, every edge $e\in\cE$, and every cycle $c\in\cC$ of the original pose graph is modeled by a node in the Bayesian network, and each edge $e$ is connected to the cycles $c$ to which it belongs.
\begin{figure}
\centering
\begin{tikzpicture}
[
  node distance=1cm and 0cm,
  simple/.style={draw, circle, thick, text width=.3cm, align=center},
  observed/.style={simple, fill=gray},
  edge arrow/.style={-latex,thick}
]

\def\xshz{12 mm}
\def\xshx{10 mm}
\def\ysh{-20 mm}

\node (z1) [simple] at (0,0) {};
\node (z2) [simple] at (\xshz,0) {};
\node (z3) [simple] at (2*\xshz,0) {};
\node (z4) [simple] at (3*\xshz,0) {};
\node (zdot) at (4*\xshz,0) {$\vct\dots$};
\node (zm) [simple] at (5*\xshz,0) {};
\node (ze)[align=left] at (6*\xshz,0) {$\bm x$};

\node (x1) [observed] at (0,\ysh) {};
\node (x2) [observed] at (\xshx,\ysh) {};
\node (x3) [observed] at (2*\xshx,\ysh) {};
\node (x4) [observed] at (3*\xshx,\ysh) {};
\node (x5) [observed] at (4*\xshx,\ysh) {};
\node (xdot) at (5*\xshx,\ysh)  {$\vct\dots$};
\node (xl) [observed] at (6*\xshx,\ysh)  {};
\node (xc)[align=left] at (ze |- xl) {$\bm z$};

\begin{scope}[edge arrow]
\draw	(z1) -- (x1);
\draw	(z1) -- (x2);
\draw	(z1) -- (x4);
\draw	(z2) -- (x1);
\draw	(z2) -- (x2);
\draw	(z2) -- (x3);
\draw	(z2) -- (x5);
\draw	(z3) -- (x1);
\draw	(z3) -- (x2);
\draw	(z3) -- (x3);
\draw	(z3) -- (x4);
\draw	(z3) -- (x5);
\draw	(z4) -- (x2);
\draw	(z4) -- (x3);
\draw	(z4) -- (x4);
\draw	(z4) -- (x5);
\draw	(zm) -- (xl);
\end{scope}
\end{tikzpicture}

\caption{Bayesian network of the edges in a pose graph and its cycles. The upper nodes correspond to the edges and the bottom nodes correspond to cycles. The bottom nodes are shown in gray since the random variable assigned to them $z_c$ is known.}
\label{Fig-Bayesian}
\end{figure}
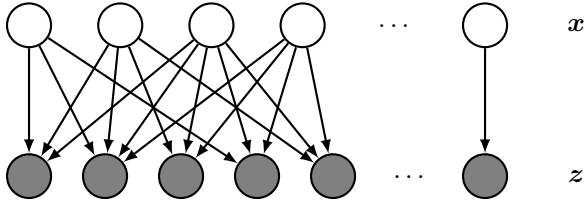

The joint probability distribution given by this graphical model for hidden states $\bm{x}\in \{0,1\}^{|\cE|}$ and cycle-consistency measurement errors $\bm{z}\in\real{|\bar{\cC}|}$ is
\begin{equation}
p(\bm x, \bm z) = \prod_{e\in\cE} p(x_e) \prod_{c\in\bar{\cC}} p(z_c \, | \, \bm x_c)
\label{Eq-GM}
\end{equation}

where $\bar{\cC}$ is the set of all cycles in $\cG$ and $p(x_e)$ is the prior probability of edge $e$, and $\bm x_c$ is the vector containing $x_e$ values for every $e\in c$. Letting $s=\bm 1^\trp\bm x_c $ be the number of outliers in $c$ for the configuration $\bm x$, the distribution $p(z_c \, | \, \bm x_c)$ is obtained from \eqref{eq:prod-distribution}, where the covariance matrix given by $\varsigma_c^2(\bm x_c) \mI_3 = \big( s \bar\sigma^2 + (|c|-s)\sigma^2 \big)  \vct I_3$ where $|c|$ is the length of the cycle.

As was mentioned in the introduction, we only need to consider loop closure edges in $\cG$. The ego motion edges contain no outlier measurements, hence we set the priors $p(x_e=0)=\pi_e$ to one for any ego motion edge $e$. 
Moreover, using all possible cycles is neither necessary nor practical for this task. The total number of cycles in a graph, in general, grows combinatorially with the size of the graph, leading to a proportional increase in the computational cost. To deal with this issue, we restrict ourselves to cycles from a \emph{Minimum Cycle Basis} $\cC_{min}\in2^{\bar{\cC}}$ of the pose graph obtained using the de Pina's method \cite{mehlhorn2007implementing}. This reduces the number of cycles to $\cO(|\cE_{lc}|)$, covers all the edges in bi-connected components of the pose-graph, and any every other cycle can be obtained as a combination of cycles in the basis. Moreover, the MCB, which is minimal in the sense of the number of times each edge appears in cycles in $\cC_{min}$, has the further side effect of reducing the number of connections in the Bayesian graphical model of \ref{Fig-Bayesian}. Moreover, from the discussion in Section \ref{sec:edge-uncertainty-model}, short cycles reduce the uncertainty in the observations $z_c$ along cycles with only inliers (in future work, we will explore the option of finding a basis that is minimal in the sense of the sum of the errors $z_c$).



\section{Inference}
In this section, we assume that the set of parameters $\Theta=\{\sigma,\bar\sigma,\Pi\}$ where $\Pi=\{\pi_e\}_{e\in\cE_{lc}}$ is given, and that we aim to find the marginal probabilities $\gamma_e\triangleq p(x_e|\vct z)$ for all $e\in\cE$ (in Section~\ref{sec:expectation-maximization}, we will extend the procedure to estimate $\sigma$,$\bar\sigma$ concurrently from the data). An exact solution to this probabilistic inference problem can easily become intractable as the complexity increases exponentially with the number of edges. Therefore, we will resort to approximation methods. We consider two options: first, we will apply Loopy Belief Propagation (BP), which represents the standard traditional choice for approximate inference in graphs, despite the fact that it does not guarantees convergence for general graphs; then, we will introduce a novel inference algorithm based on dual decomposition along cycles with the Alternating Direction Method of Multipliers (ADMM), which provides local convergence guarantees. It is shown in Section~\ref{sec:results} that the our proposed ADMM method outperforms BP in terms of outliers detection in our setting.
\subsection{Belief Propagation}
Belief Propagation is one of the most well known inference algorithms used for finding marginal and conditional probabilities, and is essentially a Variational Inference approach based on the minimization of the Bethe free energy \cite{yedidia2005constructing}. For graphical models with loops, BP usually provides a good estimate but may not converge. Even if converges, the given solution is generally not exact. We review here the factor graph version of BP via an example shown in Figure~\ref{Fig-BayesianExample}, where we give a pose graph with five edges and a total of three cycles, together with the corresponding Bayesian network and the equivalent factor graph. 

In BP, messages are sent between neighboring variables and factors according to the following equations \cite{yedidia2005constructing}:

\begin{align}
n_{e\rightarrow f_c}(x_e) &= \prod_{f\in N(e)\setminus f_c} m_{f\rightarrow e}(x_e), \label{Eq-VarToFactor}\\
m_{f_c\rightarrow e}(x_e) &=\sum_{\bm x_c\setminus x_e} f_c(\bm x_c) \prod_{i\in N(f_c)\setminus e} n_{i\rightarrow f_c}(x_i), \label{Eq-FactorToVar}
\end{align}
where \eqref{Eq-VarToFactor} shows the message sent from variable $e$ to factor $f_c$, and \eqref{Eq-FactorToVar} shows the message sent from factor $f_c$ to variable $e$. The notation $N(e)$ means the factors that are connected to random variable $x_e$ (including the prior $\pi_e$ which is constant) and $N(f_c)$ is the random variables connected to $f_c$. These messages are passed in an asynchronous order until convergence of beliefs (approximate marginals), which are computed from the equations \cite{yedidia2005constructing}:

\begin{align}
b_e(x_e) &\propto \prod_{f\in N(e)} m_{f\rightarrow e}(x_e), \label{Eq-BeliefVar} \\
b_c(\bm x_c) &\propto f_c(\bm x_c) \prod_{e\in N(f_c)} n_{e\rightarrow f_c}(x_e), \label{Eq-BeliefFactor}
\end{align}
where \eqref{Eq-BeliefVar} is the belief of a single random variable and is an approximation of $\gamma_e=p(x_e|\bm z)$, and \eqref{Eq-BeliefFactor} is the belief of all random variables connected to factor $f_c$ and an approximation of $\gamma_c\triangleq p(\bm x_c|\bm z)$ (the latter is use $\gamma_c$ only in for the Expectation-Maximization procedure in Section~\ref{sec:results}).






\begin{figure*}
\subfloat[Pictorial representation]
{\begin{tikzpicture}[scale=0.55]
  \input{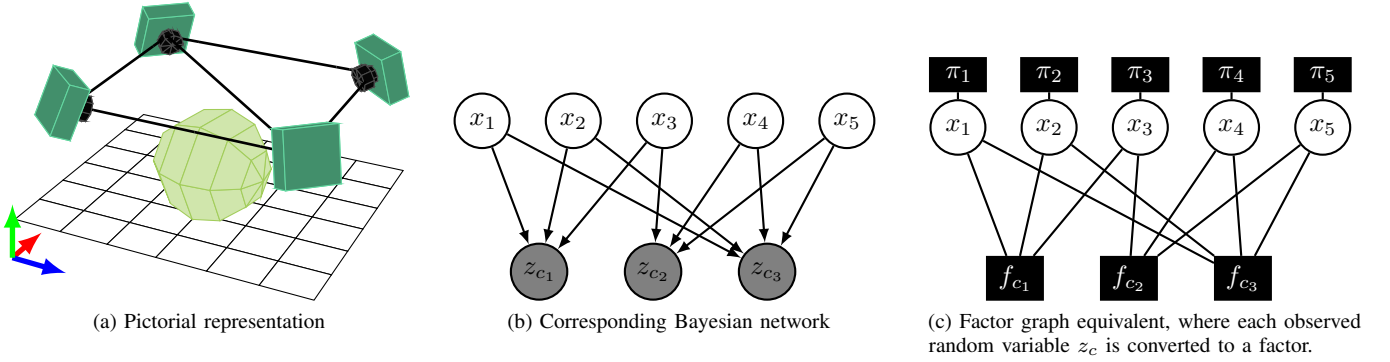} 
\end{tikzpicture}}\hfill
\subfloat[Corresponding Bayesian network]{\begin{tikzpicture}
[
  node distance=1cm and 0cm,
  simple/.style={draw, circle, thick, text width=.5cm, align=center, inner sep=2pt},
  observed/.style={simple, fill=gray},
  edge arrow/.style={-latex,thick}
]

\def\xshz{12 mm}
\def\xshx{15 mm}
\def\ysh{-20 mm}

\node (z1) [simple] at (0,0) {$x_1$};
\node (z2) [simple] at (\xshz,0) {$x_2$};
\node (z3) [simple] at (2*\xshz,0) {$x_3$};
\node (z4) [simple] at (3*\xshz,0) {$x_4$};
\node (z5) [simple] at (4*\xshz,0) {$x_5$};

\node (x1) [observed] at (0.5*\xshx,\ysh) {$z_{c_1}$};
\node (x2) [observed] at (1.5*\xshx,\ysh) {$z_{c_2}$};
\node (x3) [observed] at (2.5*\xshx,\ysh) {$z_{c_3}$};

\begin{scope}[edge arrow]
\draw	(z1) -- (x1);
\draw	(z2) -- (x1);
\draw	(z3) -- (x1);

\draw	(z3) -- (x2);
\draw	(z4) -- (x2);
\draw	(z5) -- (x2);

\draw	(z1) -- (x3);
\draw	(z2) -- (x3);
\draw	(z4) -- (x3);
\draw	(z5) -- (x3);
\end{scope}
\end{tikzpicture}
}\hfill
\subfloat[Factor graph equivalent, where each observed random variable $z_c$ is converted to a factor.]{\begin{tikzpicture}
[
  node distance=1cm and 0cm,
  simple/.style={draw, circle, thick, text width=.5cm, align=center, inner sep=2pt},
  observed/.style={simple, fill=gray},
  edge arrow/.style={-latex,thick}
]

\tikzstyle{observed}=[draw=black, rectangle, thick, align=center, fill=black, text=white, text width=.5cm]

\def\xshz{12 mm}
\def\xshx{15 mm}
\def\ysh{-20 mm}
\def\yshh{-7 mm}

\node (p1) [observed] at (0,-\yshh) {$\pi_1$};
\node (p2) [observed] at (\xshz,-\yshh) {$\pi_2$};
\node (p3) [observed] at (2*\xshz,-\yshh) {$\pi_3$};
\node (p4) [observed] at (3*\xshz,-\yshh) {$\pi_4$};
\node (p5) [observed] at (4*\xshz,-\yshh) {$\pi_5$};

\node (z1) [simple] at (0,0) {$x_1$};
\node (z2) [simple] at (\xshz,0) {$x_2$};
\node (z3) [simple] at (2*\xshz,0) {$x_3$};
\node (z4) [simple] at (3*\xshz,0) {$x_4$};
\node (z5) [simple] at (4*\xshz,0) {$x_5$};

\node (x1) [observed] at (0.5*\xshx,\ysh) {$f_{c_1}$};
\node (x2) [observed] at (1.5*\xshx,\ysh) {$f_{c_2}$};
\node (x3) [observed] at (2.5*\xshx,\ysh) {$f_{c_3}$};

\begin{scope}[thick]
\draw	(z1) -- (x1);
\draw	(z2) -- (x1);
\draw	(z3) -- (x1);

\draw	(z3) -- (x2);
\draw	(z4) -- (x2);
\draw	(z5) -- (x2);

\draw	(z1) -- (x3);
\draw	(z2) -- (x3);
\draw	(z4) -- (x3);
\draw	(z5) -- (x3);

\draw	(z1) -- (p1);
\draw	(z2) -- (p2);
\draw	(z3) -- (p3);
\draw	(z4) -- (p4);
\draw	(z5) -- (p5);
\end{scope}

\end{tikzpicture}

\label{Fig-FactorExample}}
\caption{Example derivation of the factor graph for a small problem with four poses, five measurements, and three cycles}
\label{Fig-BayesianExample}
\end{figure*}

To force BP into convergence, we use a damping factor of $0.5$, as suggested in \cite[Chapter~22]{robert2014machine}.

\subsection{Alternating Direction Method of Multipliers}
The Alternating Direction Method of Multipliers (ADMM) provides a robust and decomposable algorithm for minimizing a convex problem by breaking them into smaller and easier to handle problems \cite{boyd2011distributed}. In this setting, ADMM guarantees global convergence at rate $\cO(\frac{1}{\delta})$ rate ($\delta$ is error) \cite{nishihara2015general}. It can also be used in non-convex problems, although in that case it will convergence to a local minimum.

In order to estimate $\gamma_e$ and $\gamma_c$, instead of marginalizing over $p(\bm x,\bm z)$ given in \eqref{Eq-GM}, we propose to marginalize over the local distribution of each cycle, namely,

\begin{equation}
p_c(\bm x_c,z_c)=p(z_c|\bm x_c)\prod_{e\in c} p(x_e),
\label{eq:localdist}
\end{equation}
and then force the marginals of each edge $e$ obtained from different cycles to agree on a common value. Intuitively, this strategy aims to preserve the statistical correlation (joint distribution) between edges in the same cycle, but it ignores the correlations across cycles.

More in detail, we can implement this strategy by solving a \emph{consensus} problem with ADMM \cite[Chapter 7]{boyd2011distributed}. 
We denote as $\hat{\bm v}_c\in\real{2^{|c|}}$ the vector containing all probabilities $p_c(\bm x_c|z_c)$ obtained from \eqref{eq:localdist} evaluated over all possible values of $\bm x_c\in\{0,1\}^{|c|}$. For each cycle $c$, we try to estimate a vector $\bm v_c$ such that \begin{enumerate*} \item $\bm v_c$ is close to  $\hat{\bm v}_c$, and \item when two distributions $\bm v_c$, $\bm v_{c'}$ for two overlapping cycles $c,c'$ are marginalized with respect to a common edge $e\in (c\cap c')$, the two results agree
\end{enumerate*}. We will parametrize the marginal distribution $\gamma_e$ by keeping track of the inlier probability alone, denoted as $w_e=p(x_e=0|\bm z)$. 
We can then formulate the following minimization problem:

\begin{equation}
\begin{aligned}
  \min_{\bm w, \{\bm v_c\}}& \ \sum_{c\in\cC} h_{c}\big(\bm v_c\big)\\
\text{subject to}& \ \vct p_{e,c}^\trp \bm v_{c} = w_e, \forall c\in \cC, e\in c\\
 &\ 0 \leq \bm w \leq 1
\end{aligned}
\label{eq:admm_formulation}
\end{equation}

In the above equation, $\bm w\in\real{|\cE_{lc}|}$ is the vector that contains all $w_e$ values, and the indicator vectors $\vct p_{e,c}\in\{0,1\}^{2^{|c|}}$ are a vectorial representation for obtaining the marginal inlier probability $w_e$ given the cycle distribution~$\bm v_c$. 

In \eqref{eq:admm_formulation}, each $h_c$ is considered a subproblem with its own local constraints that can be solved in a distributed fashion. As stated earlier, we want $\bm v_c$ to be close to $\hat{\bm v_c}$ with respect to some metric. If we choose the $2$-Wasserstein metric, $h_c$ will be formulated as follows:

\begin{equation}
  h_c\big(\bm v_c\big)=\left\{
\begin{array}{ll}
{\norm{\bm v_c-\hat{\bm v}_c}^2} & {\text {if } \bm 1^\trp \bm v_c = 1, 0\leq\bm v_c\leq1 ,} \\ 
{+\infty} & {\text {otherwise.}} 
\end{array}\right.
\end{equation}

In future work, we plan to evaluate other measures of similarity between $ c$ and $\hat{ c}$ (such as the Kullback–Leibler divergence).
Subproblems (cycles) $c$, $c'$ that share an edge are forced to agree through the constraints $\vct p_{e,c}^\trp \bm v_{c} =\vct p_{e,c'}^\trp \bm v_{c'}= w_e$. This problem formulation is very similar to a consensus optimization problem, the only difference is that we want a linear combination of the variables $\bm v_{c}$ to reach a consensus instead of considering the full variables $\bm v_c$. There is also a global constraint $0\leq \bm w \leq 1$ that needs to be satisfied. To apply ADMM, we reformulate the problem as follows,

\begin{equation}\label{equ:consensusAdmm}
\begin{aligned}
\min_{w_e} &\sum_{c\in\cC} h_c\big(\bm v_c\big) + g(\bm w)\\
\textrm{subject to } & {{ \mP_{c}} \bm v_c = \bm w_c,\ \forall c\in \cC}\\
\end{aligned}
\end{equation}

where $g(\bm w)$ is a indicator function which returns $+\infty$ if the constraint $0\leq\bm w\leq1$ is violated; the vector $\bm w_c\in\real{|c|}$ contains the elements $w_e$ of $\bm w$ for every $e\in c$ and $\mP \in \mathbb{R}^{|c|\times 2^{|c|}}$ is obtained by stacking the vectors $\vct p_{e,c}^\trp$  column-wise. The augmented Lagrangian for (\ref{equ:consensusAdmm}) is:
\begin{equation}
\begin{aligned}
L_{\rho} =&  \sum_{c\in\cC}\Big(h_c(\bm v_c) + \bm y_c^\trp(\mP_c \bm v_c - \bm w_c) \\
		    &+ \frac{\rho}{2} \|\mP_c \bm v_c - \bm w_c\|^2_2\Big) + g(\bm w)
\end{aligned}
\end{equation}
with dual variables $\bm y_c \in \real{|c|}$, and penalty parameter $\rho$. The ADMM iterations for this problem are given by:
\begin{equation}
\begin{aligned}
\bm v_c^{k+1} &:= \argmin_{\bm v_c}\Big( h_c(\bm v_c)+\bm y_c^{k\trp} {\mP_c} \bm v_c\\
 			   & \quad   +\frac{\rho}{2}\|{\mP_c}{\bm v_c}- {\bm w}_c^k\|^2_2\Big)\\
\bm w^{k+1} &:=  \argmin_{\bm w} \Big(g(\bm w)+\sum_{c\in\cC}\big( -\bm y_c^{k\trp} {\bm w}_c\\
 			       &\quad+ \frac{\rho}{2}\|{\mP_c} \bm v_c- {\bm w}_c\|^2_2 \big)\Big)\\
\bm y_c^{k+1} &:= \bm y_c^k + \rho(\mP_c \bm v_c- {\bm w}_c^k), 
\end{aligned}
\end{equation}
The local variables $\bm v_{c}$ and $\bm y_c$ are solved individually for each cycle. The solution for $\bm v_{c}$ is obtained by solving a quadratic programming problem (which can be done efficiently), and the solution for the global consensus variable $\bm w$ can be simplified as:
\begin{equation}
w^{k+1}_e := \max(0, \min(1, \omega^{k+1}_e))
\label{eq:admmave1}
\end{equation}
\begin{equation}
\omega^{k+1}_e = \frac{\sum_{c; e\in c}\left(\vct p_{e,c}^\trp\bm v^{k+1}_c+\frac{1}{\rho}(\bm y^k_c)_e \right)}{\sum_{c; e\in c} 1}
\label{eq:admmave2}
\end{equation}
In \eqref{eq:admmave2}, the denominator is the number of times edge $e$ appears in different cycles, and therefore $\omega^{k+1}_e$ is the average of marginalized values for edge $e$ plus the component of $\bm y_c^k$ that corresponds to $e$ over cycles that contain $e$. In \eqref{eq:admmave1}, the values of $ \omega^{k+1}$ are projected to be between zero and one.

This problem will reach optimality when the primal residual $r^k$ and dual residuals $t^k$ converge to zero, where:
\begin{equation}
\begin{aligned}
r^{k}&=\sum_{c\in\cC}\left\|\mP_c \bm v_c^k-{\bm w}_c^k\right\|_{2}^{2}\\
t^{k}&= \rho^{2}\sum_{e\in\cE_{lc}}\sum_{c; e\in c} (w_e^{k}- w_e^{k-1})^{2}
\end{aligned}
\end{equation}
The penalty parameter $\rho$ plays a very important role in the convergence speed of this method. Intuitively, small $\rho$ allows intermediate solutions to have a much lower cost while somewhat ignoring the primal feasibility, and makes the solution less impacted by initial value and easier to escape from the local minima, whereas a large $\rho$ will place a large penalty on violating the consensus constraints, but tends to produce small primal residuals. As suggested in \cite[Chapter 3]{boyd2011distributed}, we start with a small $\rho$, and gradually change the value of $\rho$ based on primal and dual residual, using the following dynamic update rule:
\begin{equation}
\rho^{k+1}=\left\{
\begin{array}{ll}
{\tau^{incr}\rho^{k}} & {\text{if $r^k\leq \mu t^k$}} \\ 
{\rho^{k}/\tau^{decr}} & {\text{if $t^k\leq \mu r^k$}} \\
{\rho^{k}}	& {\text{otherwise.}}
\end{array}\right.
\end{equation}

where $\mu >1$, $\tau^{decr}>1$ and $\tau^{decr}>1$ are constant parameters. 


\section{Expectation Maximization}
\label{sec:expectation-maximization}
In the previous sections, we assumed that the parameters $\Theta=\{\sigma,\bar\sigma,\Pi\}$ were given. However, this assumption is not true and these parameters need to be estimated.
By including parameters in the distribution, we rewrite \eqref{Eq-GM} as:
\begin{equation}
p(\bm x, \bm z|\Theta) = \prod_{e\in\cE} p(x_e |\pi_e) \prod_{c\in\cC} p(z_c \, | \, \varsigma_c(\bm x_c))\end{equation}

where the first term is a given by Bernoulli distribution. With a little abuse of notation, we assume $\pi_e$ is $p(x_e=0)$ and $\bar\pi_e=1-\pi_e$ which yields $p(x_e|\pi_e)=\pi_e^{1-x_e}\bar\pi_e^{x_e}$. The second term is a wrapped Gaussian mixture distribution:
\begin{equation}
p(z_c \, | \, \varsigma_c(\bm x_c)) = \frac{1}{\psi_c} \frac{\varsigma_c^{-3}}{\phi(\varsigma_c)}\exp(\frac{-z_c^2}{2\varsigma_c^{2}})
\end{equation}
with $\varsigma_c(\bm x)=\sqrt{(\vct 1^\trp \bm x_c)\bar\sigma^2+(|c|-\vct 1^\trp \bm x_c)\sigma^2}$ and $\phi(\varsigma_c)$ is a normalizing constant which normalizes the wrapped Gaussian distribution and $\psi_c$ normalizes over all possible configurations of $\bm x$:
\begin{equation}
\psi_c = \sum_{s=0}^{|c|}\binom{|c|}{s} \frac{\varsigma_c^{-3}(s)}{\phi(\varsigma_c(s))}\exp(\frac{-z_c^2}{2\varsigma_c^{2}(s)})
\end{equation}
and the term $\varsigma_c^{-3}$ comes from the denominator of the Gaussian probability density function, which is $\sqrt{\det(\varsigma_c^2\mI_3)}$. The value of $\varsigma_c(\bm x)$ only depends on the number of outliers $s=\vct 1^\trp \bm x$, hence we can denote it is $\varsigma_c(s)$.
The log-likelihood function is given by:
\begin{equation}
\begin{aligned}
\cL(\Theta;\bm x,\bm z) =& \log(p(\bm x, \bm z|\Theta) ) \\
 =& \sum_{e\in\cE} (1-x_e)\log(\pi_e) + x_e\log(\bar\pi_e) \\
+& \sum_{c\in\cC} -3\log\varsigma_c - \frac{z_c^2}{2\varsigma_c^2} -\log\big(\psi_c \phi(\varsigma_c) \big)
\end{aligned}
\end{equation}
In the Expectation step, we find the expectation of the log likelihood function of $\Theta^i$ with respect to the current distribution of $\bm x$ given $\bm z$ and previous estimate of parameters $\Theta^{i-1}$:
\begin{equation}
\begin{aligned}
Q(\Theta^{(i)}|\Theta^{(i-1)})&= \mathbb{E}_{\bm x | \bm z,\Theta^{(i-1)} } [\cL] \\
 &= \sum_{\bm x \in \mathbb{Z}_2^{|\cE|}} \cL(\Theta^{(i)};\bm x,\bm z)p(\bm x|\bm z,\Theta^{(i-1)})
\label{Eq-Q}
\end{aligned}
\end{equation}
We denote $\gamma_e^{(i-1)}\triangleq p(x_e|\bm z, \Theta^{(i-1)})$ as the responsibility of edge $e$ and $\gamma_c^{(i-1)}\triangleq p(\bm x_c|\bm z, \Theta^{(i-1)})$ as the responsibility of cycle $c$, estimated either through BP or ADMM using the parameter estimates from last iteration. Now, by expanding \eqref{Eq-Q} we get:
\begin{align}
Q&(\Theta^{(i)}|\Theta^{(i-1)})=  \sum_{e\in\cE} \sum_{x_e\in\mathbb{Z}_2} p(x_e|\bm z,\Theta^{(i-1)})\log p(x_e|\pi_e^{(i)}) \label{Eq-QQ1} \\
&+ \sum_{c\in\cC} \sum_{\bm{x}_c\in\mathbb{Z}_2^{|c|}} p(\bm x_c|\bm z,\Theta^{(i-1)}) \log p(z_c|\bm{x}_c,\sigma^{(i)},\bar\sigma^{(i)}) \label{Eq-QQ2}
\end{align}

In the Maximization step, we find $\Theta^{(i)} = \argmax_{\Theta} Q(\Theta | \Theta^{(i-1)})$. For $\Pi^{(i)}$, we get:
\begin{equation}
\pi_e^{(i)} = \gamma_e^{(i-1)} 
\end{equation}
but for  $\sigma^{(i)}$ and $\bar\sigma^{(i)}$ it is not as straightforward. Each term in the summation in \eqref{Eq-QQ2} is a quasiconcave function, but their sum need not be quasiconcave. Therefore, a grid-search is utilized to find $\sigma$ and $\bar\sigma$ at each iteration.

\section{Simulations and Experiments on map merging}
\label{sec:results}
In this section, we provide performance results of our outlier detection algorithm over synthetic and real data.
For the synthetic data, for every simulation a pose graph with two maps with random poses are generated where each map has 15 nodes. At every iteration, $m$ edges are added between the two maps where $m$ varies from 10 to 200 with 5 increments. For every given $m$, from 1 to $m-1$ edges are selected to be outliers (with 1 increments). Inlier and outlier edges are given a random noise rotation with a random direction, and the magnitude of noise uniformly seclected within $1.6^\circ\leq\|\bm\epsilon\|\leq 2.4^\circ$ for inliers and $16^\circ\leq\|\bar{\bm\epsilon}\|\leq 24^\circ$ for outliers. Total number of generated graphs is $8,317$ and both BP and ADMM inference algorithms were used in on the same graphs. 

In Fig. \ref{SFig-PRa}, the precision-recall points for each of these simulations are plotted. n Fig. \ref{SFig-PRb}, the ratio of detected outliers is plotted versus the ratio of the outlier edges to total loop closure edges. It is clear that ADMM inference performs better compared to BP, as an overall higher precision and recall is achieved. Also, as the ratio of outliers to loop closure edges increase, the performance of BP continuously deteriorates. But ADMM is able to perform better. In situations where nearly half of the loop closures are outliers, the performance of ADMM is at lowest, which could be because it becomes harder to distinguish between the classes. 

In Fig. \ref{Fig-SLAMpcl}\, we present the result of implementing our classifier on actual data obtained from an office environment. Four independent sequence of RGB-D images were obtained using Intel RealSense D435 camera, and were processed with ORB-SLAM2 \cite{mur2017orb}. The result of merging maps is shown with and without removing outliers. With removing outliers, the obtained pointcloud is sharper and objects are better aligned.

\begin{figure}
  \centering
  \subfloat[]{%
      \includegraphics[]{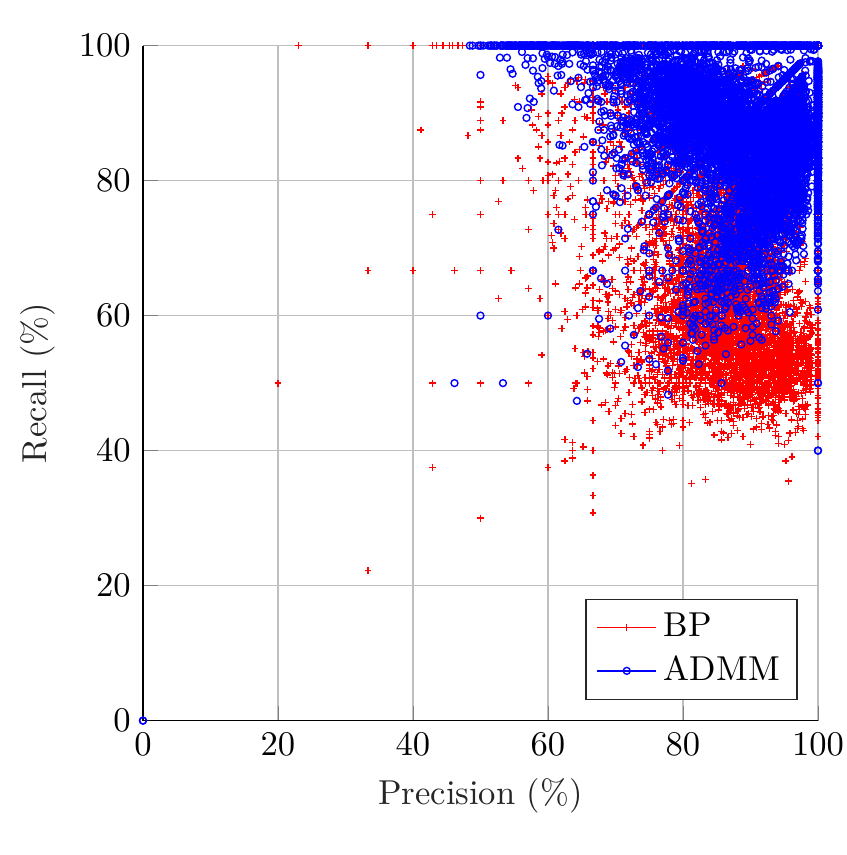}\label{SFig-PRa}
     }
     \\
     \subfloat[]{%
      \includegraphics[]{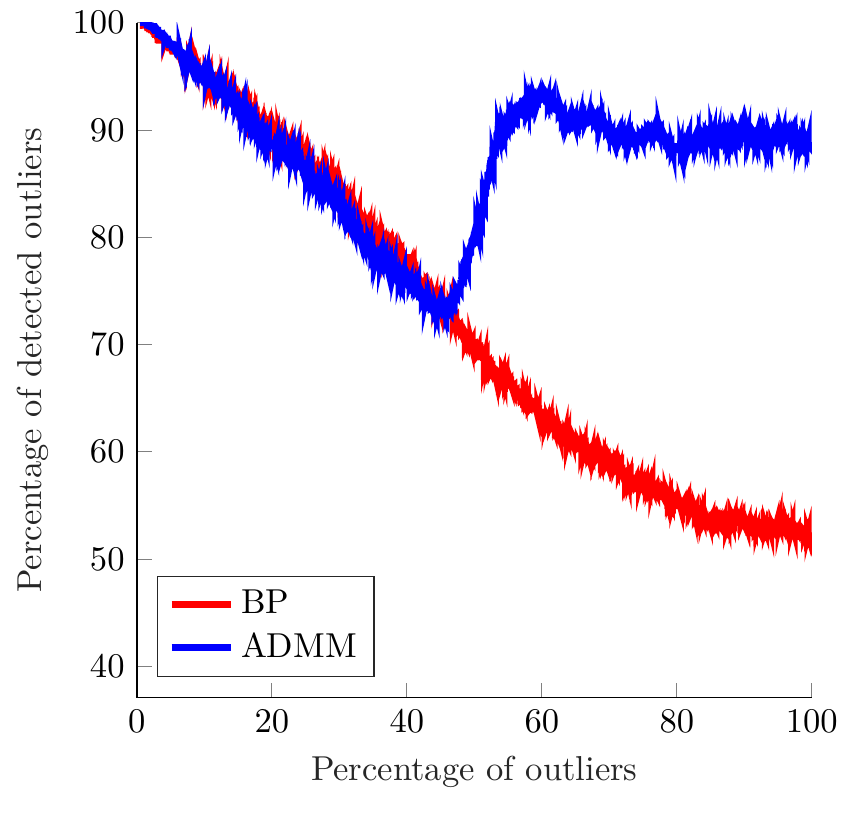}\label{SFig-PRb}
     }
\caption{Precision ($\frac{TP}{TP+FP}$)-Recall ($\frac{TP}{TP+FN}$) plot of experiments is depicted in\protect\subref{SFig-PRa}. The ratio of detected outliers to total number of outliers (Recall) versus the ratio of outlier loop closure edges to total loop closure edges is given in \protect\subref{SFig-PRb}.}
          \label{Fig-PR}
\end{figure}




\begin{figure*}
  \centering
  \subfloat[]{%
      \includegraphics[trim={0cm 0cm 0cm 0cm},clip,scale=0.11]{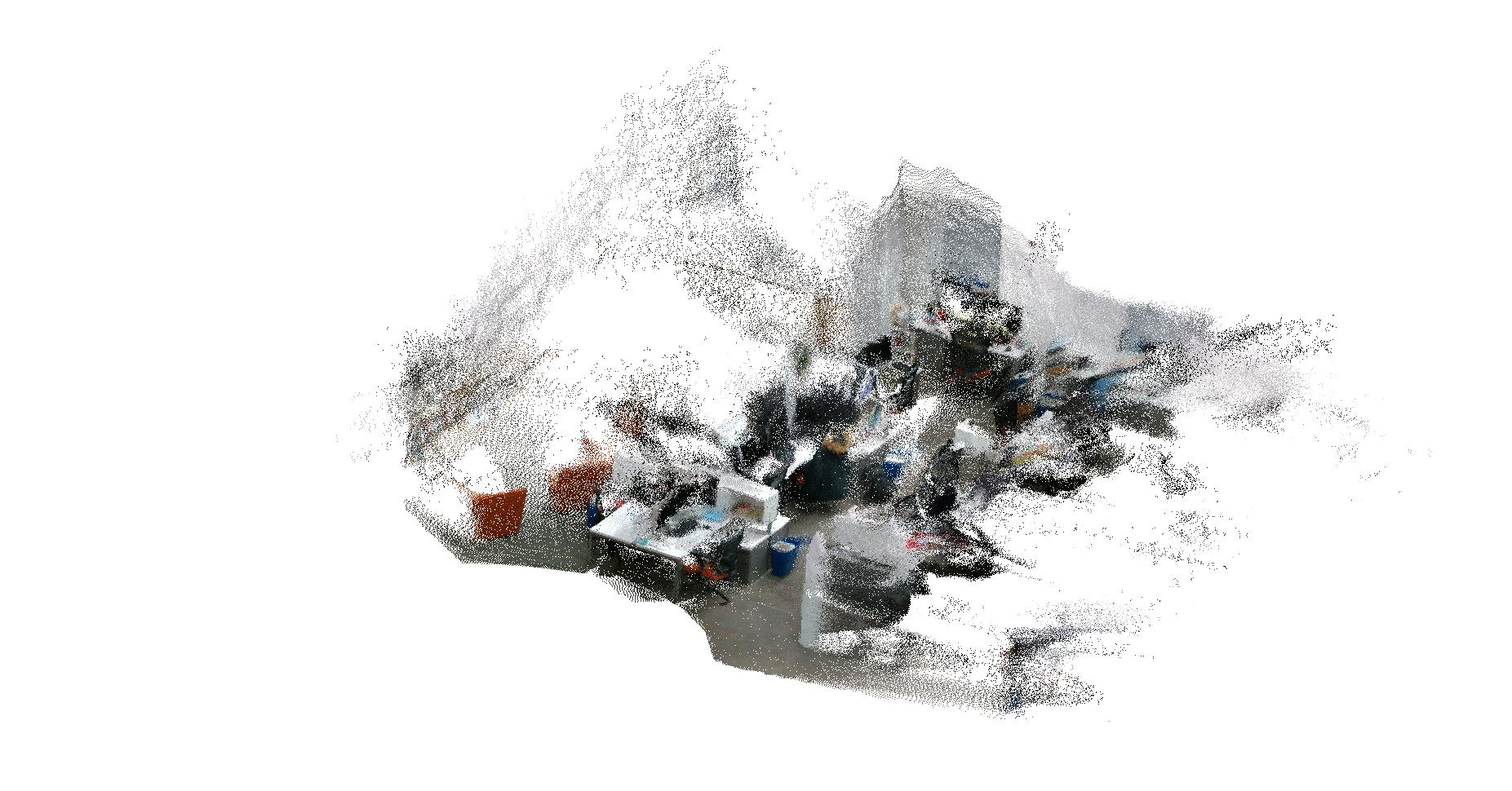}\label{SFig-SLAMa}
     }
     \hfill
     \subfloat[]{%
      \includegraphics[trim={0cm 0cm 0cm 0cm},clip,scale=0.11]{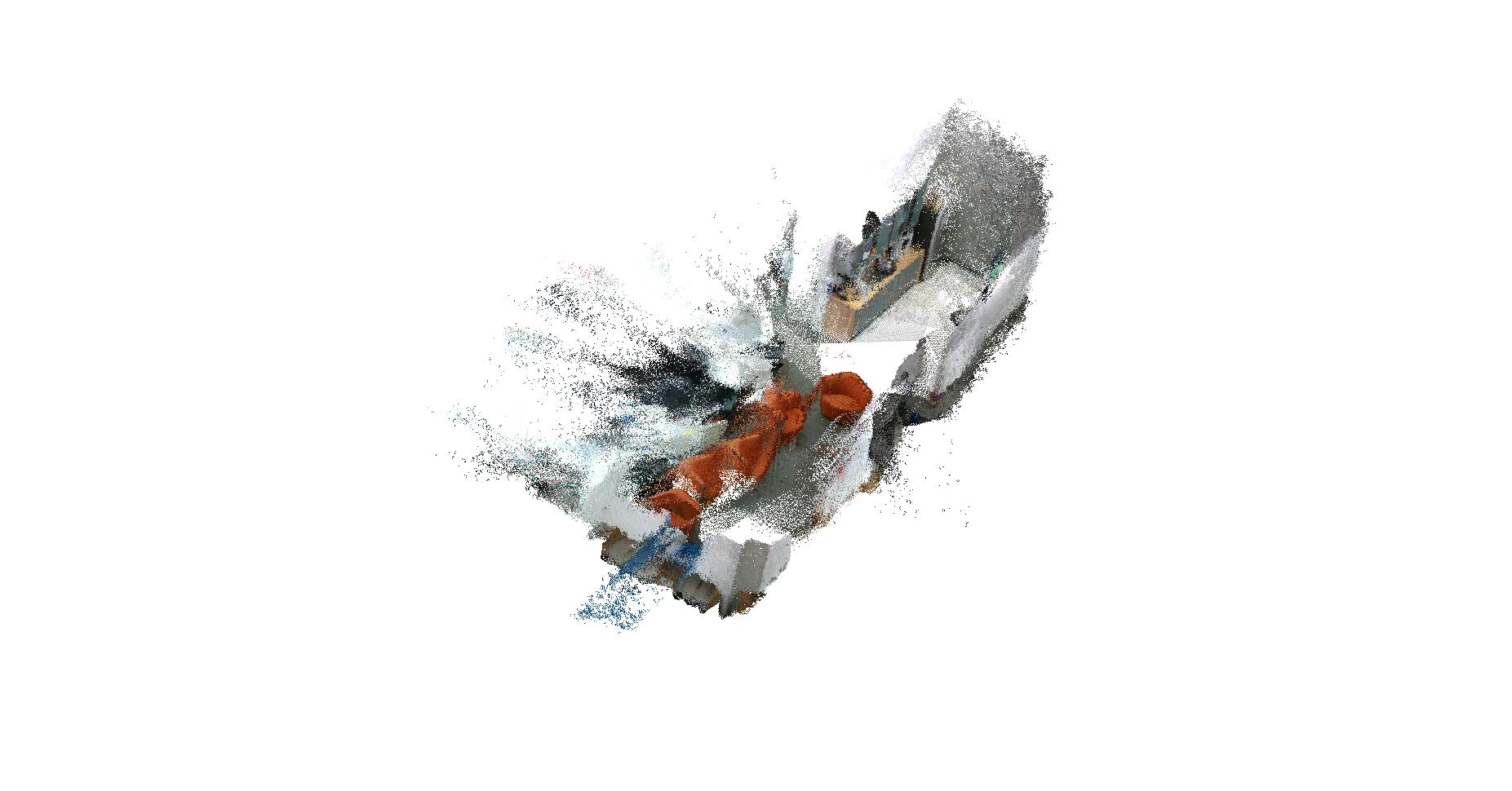}\label{SFig-SLAMb}
     }
     \\
     \subfloat[]{%
      \includegraphics[trim={0cm 0cm 0cm 0cm},clip,scale=0.11]{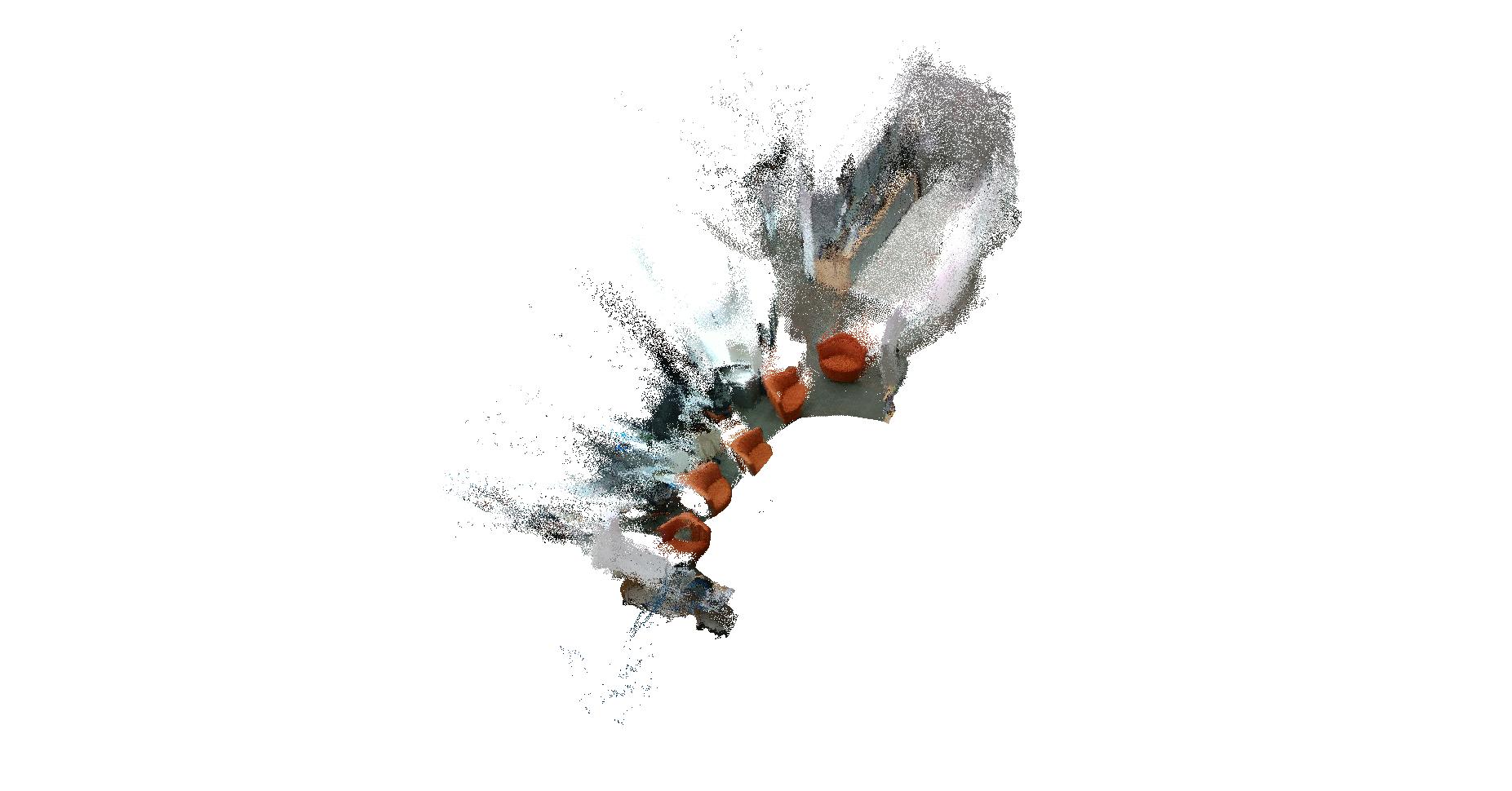}\label{SFig-SLAMc}
     }
     \hfill
     \subfloat[]{%
      \includegraphics[trim={0cm 0cm 0cm 0cm},clip,scale=0.11]{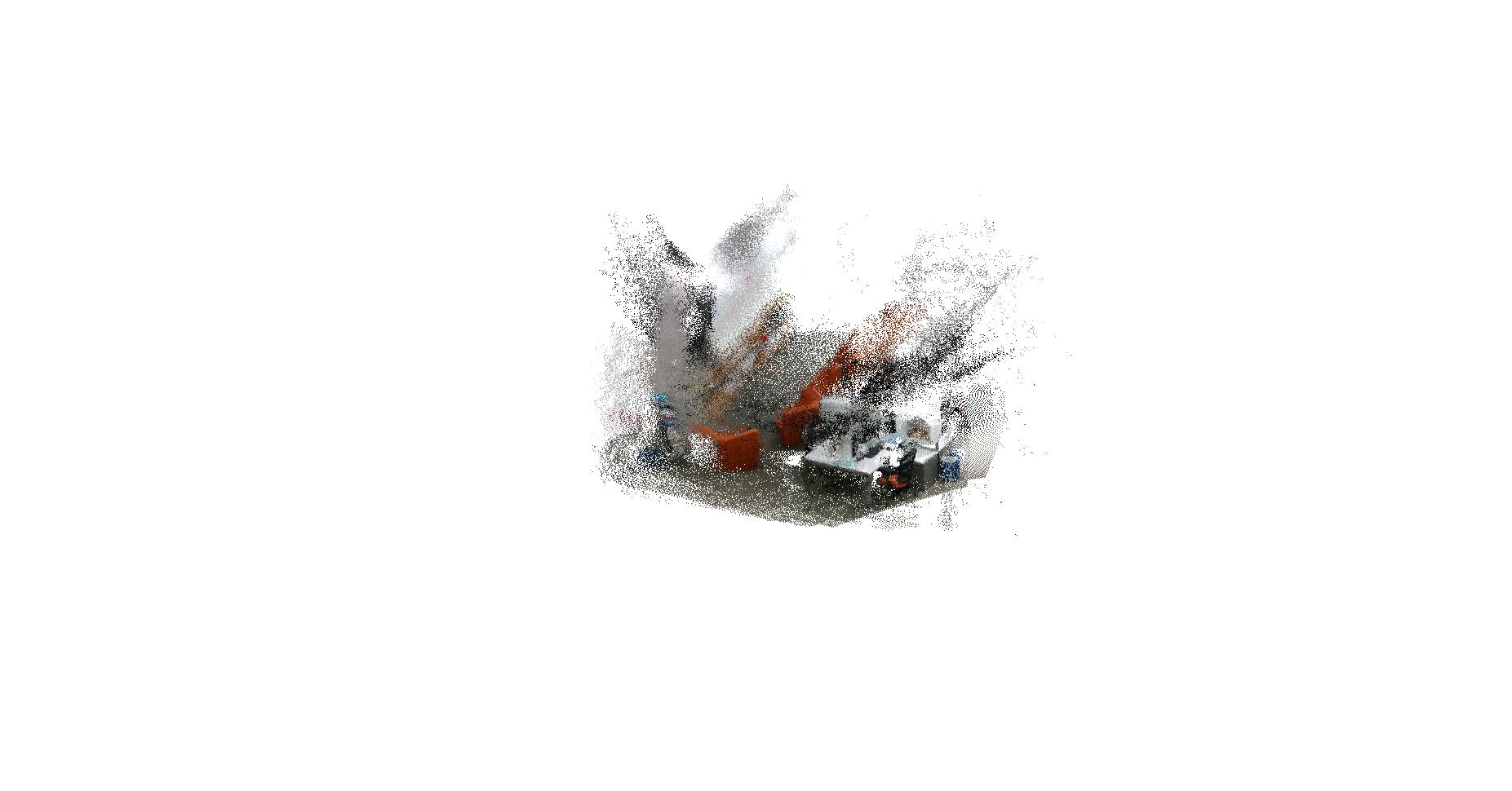}\label{SFig-SLAMd}
     }
     \\
     \subfloat[]{%
      \includegraphics[trim={0cm 0cm 0cm 0cm},clip,scale=0.14]{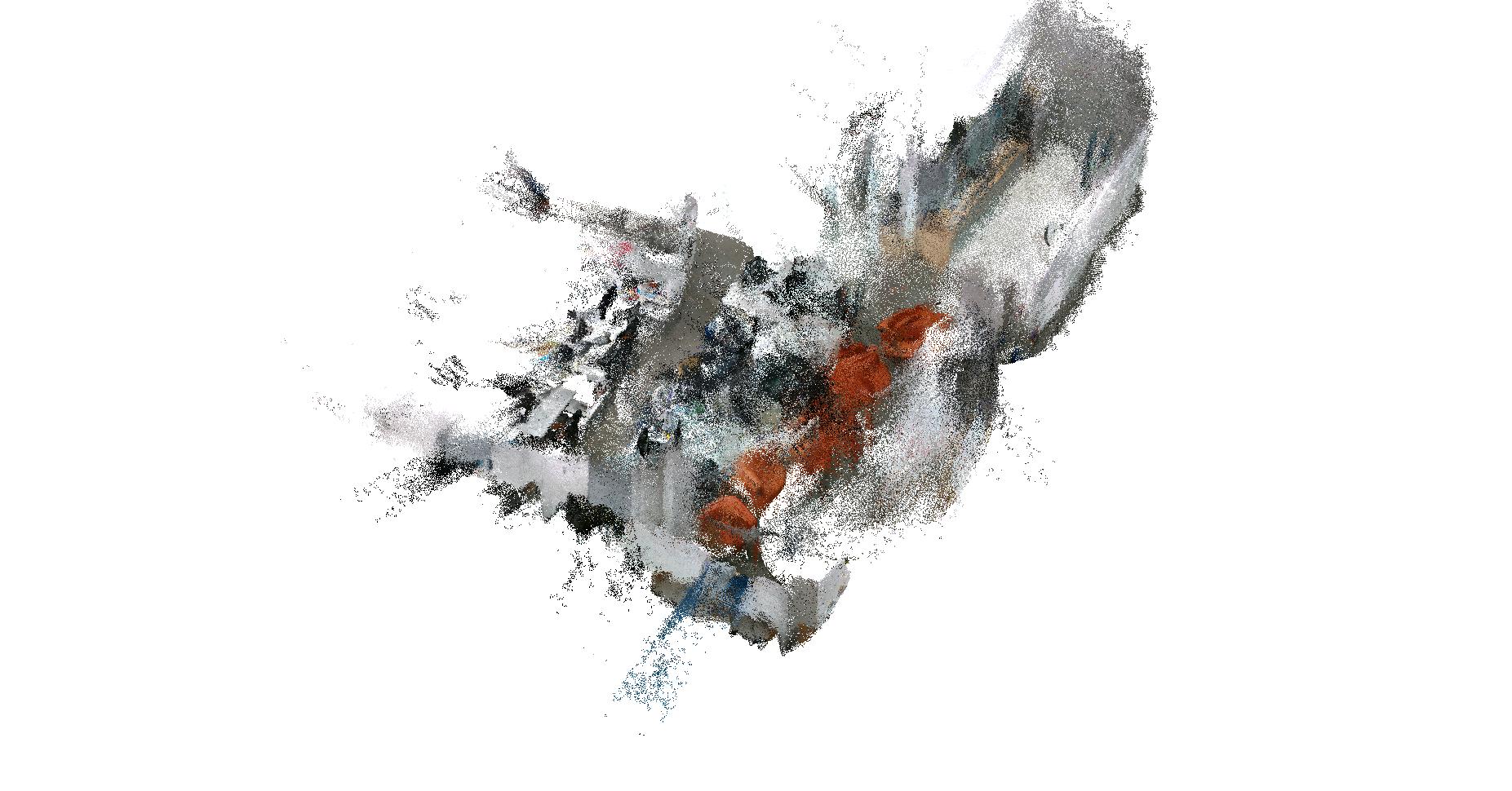}\label{SFig-SLAMe}
     }
     \\
     \subfloat[]{%
      \includegraphics[trim={0cm 0cm 0cm 0cm},clip,scale=0.14]{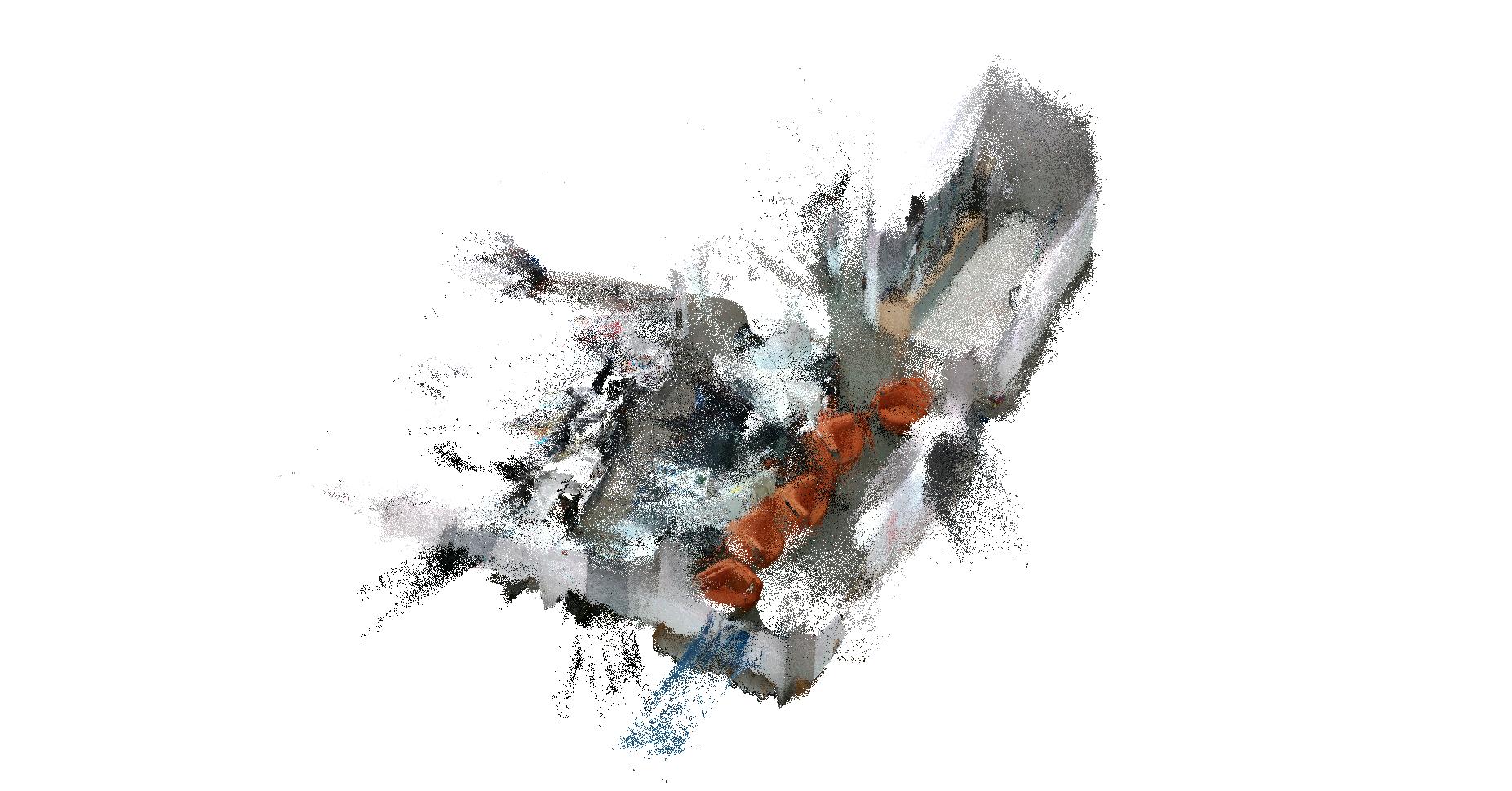}\label{SFig-SLAMf}
     }
     \caption{ In \protect\subref{SFig-SLAMa}-\protect\subref{SFig-SLAMd}, four pointclouds from four different sequence of images are obtained. Relative pose measurements between a subset of images are obtained through feature matching and the use of an object detector for place recognition. After removing the outliers and joining the maps, an optimization problem is solved to reduce the overall error. The final pointcloud made from joining all the four pointclouds without outlier detection is shown in \protect\subref{SFig-SLAMe} and with outlier detection is shown in \protect\subref{SFig-SLAMf}.}
          \label{Fig-SLAMpcl}
\end{figure*}


\section{Conclusion}
In this paper, we presented a probabilistic outlier detection algorithm which detects outliers based on the geometric consistency of rotation measurements over the cycles of a pose graph. We introduced a novel discreet inference algorithm with convergence guarantees that performed better than Belief Propagation. Every step in our algorithm except finding the minimum cycle basis can be implemented in a distributed fashion (i.e. \cite{kamran2019deco}). For our future work, we plan to make our algorithm fully distributed by relaxing the minimum criteria and use a cycle basis which is obtained with a distributed approach.  





\bibliographystyle{ieee}
\bibliography{bibroot}

\begin{thebibliography}{10}\itemsep=-1pt

\bibitem{agarwal2013robust}
P.~Agarwal, G.~D. Tipaldi, L.~Spinello, C.~Stachniss, and W.~Burgard.
\newblock Robust map optimization using dynamic covariance scaling.
\newblock In {\em 2013 IEEE International Conference on Robotics and
  Automation}, pages 62--69. Citeseer, 2013.

\bibitem{agarwal2011building}
S.~Agarwal, Y.~Furukawa, N.~Snavely, I.~Simon, B.~Curless, S.~M. Seitz, and
  R.~Szeliski.
\newblock Building rome in a day.
\newblock {\em Communications of the ACM}, 54(10):105--112, 2011.

\bibitem{agarwal2010bundle}
S.~Agarwal, N.~Snavely, S.~M. Seitz, and R.~Szeliski.
\newblock Bundle adjustment in the large.
\newblock In {\em European conference on computer vision}, pages 29--42.
  Springer, 2010.

\bibitem{bay2008speeded}
H.~Bay, A.~Ess, T.~Tuytelaars, and L.~Van~Gool.
\newblock Speeded-up robust features (surf).
\newblock {\em Computer vision and image understanding}, 110(3):346--359, 2008.

\bibitem{bosse2016robust}
M.~Bosse, G.~Agamennoni, I.~Gilitschenski, et~al.
\newblock Robust estimation and applications in robotics.
\newblock {\em Foundations and Trends{\textregistered} in Robotics},
  4(4):225--269, 2016.

\bibitem{boyd2011distributed}
S.~Boyd, N.~Parikh, E.~Chu, B.~Peleato, J.~Eckstein, et~al.
\newblock Distributed optimization and statistical learning via the alternating
  direction method of multipliers.
\newblock {\em Foundations and Trends{\textregistered} in Machine learning},
  3(1):1--122, 2011.

\bibitem{cadena2016past}
C.~Cadena, L.~Carlone, H.~Carrillo, Y.~Latif, D.~Scaramuzza, J.~Neira, I.~Reid,
  and J.~J. Leonard.
\newblock Past, present, and future of simultaneous localization and mapping:
  Toward the robust-perception age.
\newblock {\em IEEE Transactions on robotics}, 32(6):1309--1332, 2016.

\bibitem{carlone2014selecting}
L.~Carlone, A.~Censi, and F.~Dellaert.
\newblock Selecting good measurements via ? 1 relaxation: A convex approach for
  robust estimation over graphs.
\newblock In {\em Intelligent Robots and Systems (IROS 2014), 2014 IEEE/RSJ
  International Conference on}, pages 2667--2674. IEEE, 2014.

\bibitem{carlone2015initialization}
L.~Carlone, R.~Tron, K.~Daniilidis, and F.~Dellaert.
\newblock Initialization techniques for 3d slam: a survey on rotation
  estimation and its use in pose graph optimization.
\newblock In {\em 2015 IEEE international conference on robotics and automation
  (ICRA)}, pages 4597--4604. IEEE, 2015.

\bibitem{chandrasekaran2012convex}
V.~Chandrasekaran, B.~Recht, P.~A. Parrilo, and A.~S. Willsky.
\newblock The convex geometry of linear inverse problems.
\newblock {\em Foundations of Computational mathematics}, 12(6):805--849, 2012.

\bibitem{dellaert2006square}
F.~Dellaert and M.~Kaess.
\newblock Square root sam: Simultaneous localization and mapping via square
  root information smoothing.
\newblock {\em The International Journal of Robotics Research},
  25(12):1181--1203, 2006.

\bibitem{dong2015domain}
J.~Dong and S.~Soatto.
\newblock Domain-size pooling in local descriptors: Dsp-sift.
\newblock In {\em Proceedings of the IEEE conference on computer vision and
  pattern recognition}, pages 5097--5106, 2015.

\bibitem{engels2006bundle}
C.~Engels, H.~Stew{\'e}nius, and D.~Nist{\'e}r.
\newblock Bundle adjustment rules.
\newblock {\em Photogrammetric computer vision}, 2(2006), 2006.

\bibitem{enqvist2011non}
O.~Enqvist, F.~Kahl, and C.~Olsson.
\newblock Non-sequential structure from motion.
\newblock In {\em Computer Vision Workshops (ICCV Workshops), 2011 IEEE
  International Conference on}, pages 264--271. IEEE, 2011.

\bibitem{fischler1981random}
M.~A. Fischler and R.~C. Bolles.
\newblock Random sample consensus: a paradigm for model fitting with
  applications to image analysis and automated cartography.
\newblock {\em Communications of the ACM}, 24(6):381--395, 1981.

\bibitem{frahm2010building}
J.-M. Frahm, P.~Fite-Georgel, D.~Gallup, T.~Johnson, R.~Raguram, C.~Wu, Y.-H.
  Jen, E.~Dunn, B.~Clipp, S.~Lazebnik, et~al.
\newblock Building rome on a cloudless day.
\newblock In {\em European Conference on Computer Vision}, pages 368--381.
  Springer, 2010.

\bibitem{graham2015robust}
M.~C. Graham, J.~P. How, and D.~E. Gustafson.
\newblock Robust incremental slam with consistency-checking.
\newblock In {\em 2015 IEEE/RSJ International Conference on Intelligent Robots
  and Systems (IROS)}, pages 117--124. IEEE, 2015.

\bibitem{hartley2012efficient}
R.~Hartley and H.~Li.
\newblock An efficient hidden variable approach to minimal-case camera motion
  estimation.
\newblock {\em IEEE transactions on pattern analysis and machine intelligence},
  34(12):2303--2314, 2012.

\bibitem{hartley2013rotation}
R.~Hartley, J.~Trumpf, Y.~Dai, and H.~Li.
\newblock Rotation averaging.
\newblock {\em International journal of computer vision}, 103(3):267--305,
  2013.

\bibitem{hartley2003multiple}
R.~Hartley and A.~Zisserman.
\newblock {\em Multiple view geometry in computer vision}.
\newblock Cambridge university press, 2003.

\bibitem{kamran2019deco}
K.~Kamran, E.~Yeh, and Q.~Ma.
\newblock Deco: Joint computation, caching and forwarding in data-centric
  computing networks.
\newblock In {\em Proceedings of the Twentieth ACM International Symposium on
  Mobile Ad Hoc Networking and Computing}, pages 111--120, 2019.

\bibitem{kobayashi1963foundations}
S.~Kobayashi.
\newblock Foundations of differential geometry vol 1 (new york: Interscience)
  kobayashi s and nomizu k 1969.
\newblock {\em Foundations of differential geometry}, 2, 1963.

\bibitem{lajoie2019modeling}
P.-Y. Lajoie, S.~Hu, G.~Beltrame, and L.~Carlone.
\newblock Modeling perceptual aliasing in slam via discrete--continuous
  graphical models.
\newblock {\em IEEE Robotics and Automation Letters}, 4(2):1232--1239, 2019.

\bibitem{latif2013robust}
Y.~Latif, C.~Cadena, and J.~Neira.
\newblock Robust loop closing over time for pose graph slam.
\newblock {\em The International Journal of Robotics Research},
  32(14):1611--1626, 2013.

\bibitem{lee2013robust}
G.~H. Lee, F.~Fraundorfer, and M.~Pollefeys.
\newblock Robust pose-graph loop-closures with expectation-maximization.
\newblock In {\em 2013 IEEE/RSJ International Conference on Intelligent Robots
  and Systems}, pages 556--563. IEEE, 2013.

\bibitem{lowe2004distinctive}
D.~G. Lowe.
\newblock Distinctive image features from scale-invariant keypoints.
\newblock {\em International journal of computer vision}, 60(2):91--110, 2004.

\bibitem{mangelson2018pairwise}
J.~G. Mangelson, D.~Dominic, R.~M. Eustice, and R.~Vasudevan.
\newblock Pairwise consistent measurement set maximization for robust
  multi-robot map merging.
\newblock In {\em 2018 IEEE International Conference on Robotics and Automation
  (ICRA)}, pages 2916--2923. IEEE, 2018.

\bibitem{mehlhorn2007implementing}
K.~Mehlhorn and D.~Michail.
\newblock Implementing minimum cycle basis algorithms.
\newblock {\em Journal of Experimental Algorithmics (JEA)}, 11:2--5, 2007.

\bibitem{mur2017orb}
R.~Mur-Artal and J.~D. Tard{\'o}s.
\newblock Orb-slam2: An open-source slam system for monocular, stereo, and
  rgb-d cameras.
\newblock {\em IEEE Transactions on Robotics}, 33(5):1255--1262, 2017.

\bibitem{nishihara2015general}
R.~Nishihara, L.~Lessard, B.~Recht, A.~Packard, and M.~I. Jordan.
\newblock A general analysis of the convergence of admm.
\newblock {\em arXiv preprint arXiv:1502.02009}, 2015.

\bibitem{olson2013inference}
E.~Olson and P.~Agarwal.
\newblock Inference on networks of mixtures for robust robot mapping.
\newblock {\em The International Journal of Robotics Research}, 32(7):826--840,
  2013.

\bibitem{olson2006fast}
E.~Olson, J.~Leonard, and S.~Teller.
\newblock Fast iterative alignment of pose graphs with poor initial estimates.
\newblock In {\em Robotics and Automation, 2006. ICRA 2006. Proceedings 2006
  IEEE International Conference on}, pages 2262--2269. IEEE, 2006.

\bibitem{robert2014machine}
C.~Robert.
\newblock Machine learning, a probabilistic perspective, 2014.

\bibitem{snavely2006photo}
N.~Snavely, S.~M. Seitz, and R.~Szeliski.
\newblock Photo tourism: exploring photo collections in 3d.
\newblock In {\em ACM Siggraph 2006 Papers}, pages 835--846. 2006.

\bibitem{snavely2008skeletal}
N.~Snavely, S.~M. Seitz, and R.~Szeliski.
\newblock Skeletal graphs for efficient structure from motion.
\newblock In {\em 2008 IEEE Conference on Computer Vision and Pattern
  Recognition}, pages 1--8. IEEE, 2008.

\bibitem{strasdat2011double}
H.~Strasdat, A.~J. Davison, J.~M. Montiel, and K.~Konolige.
\newblock Double window optimisation for constant time visual slam.
\newblock In {\em Computer Vision (ICCV), 2011 IEEE International Conference
  on}, pages 2352--2359. IEEE, 2011.

\bibitem{sunderhauf2012switchable}
N.~S{\"u}nderhauf and P.~Protzel.
\newblock Switchable constraints for robust pose graph slam.
\newblock In {\em 2012 IEEE/RSJ International Conference on Intelligent Robots
  and Systems}, pages 1879--1884. IEEE, 2012.

\bibitem{sunderhauf2012towards}
N.~S{\"u}nderhauf and P.~Protzel.
\newblock Towards a robust back-end for pose graph slam.
\newblock In {\em 2012 IEEE International Conference on Robotics and
  Automation}, pages 1254--1261. IEEE, 2012.

\bibitem{taketomi2017visual}
T.~Taketomi, H.~Uchiyama, and S.~Ikeda.
\newblock Visual slam algorithms: A survey from 2010 to 2016.
\newblock {\em IPSJ Transactions on Computer Vision and Applications}, 9(1):16,
  2017.

\bibitem{triggs1999bundle}
B.~Triggs, P.~F. McLauchlan, R.~I. Hartley, and A.~W. Fitzgibbon.
\newblock Bundle adjustment?a modern synthesis.
\newblock In {\em International workshop on vision algorithms}, pages 298--372.
  Springer, 1999.

\bibitem{tron2016survey}
R.~Tron, X.~Zhou, and K.~Daniilidis.
\newblock A survey on rotation optimization in structure from motion.
\newblock In {\em Proceedings of the IEEE Conference on Computer Vision and
  Pattern Recognition Workshops}, pages 77--85, 2016.

\bibitem{wang2008nonparametric}
Y.~Wang and G.~S. Chirikjian.
\newblock Nonparametric second-order theory of error propagation on motion
  groups.
\newblock {\em The International journal of robotics research},
  27(11-12):1258--1273, 2008.

\bibitem{yedidia2005constructing}
J.~S. Yedidia, W.~T. Freeman, and Y.~Weiss.
\newblock Constructing free-energy approximations and generalized belief
  propagation algorithms.
\newblock {\em IEEE Transactions on information theory}, 51(7):2282--2312,
  2005.

\bibitem{zach2010disambiguating}
C.~Zach, M.~Klopschitz, and M.~Pollefeys.
\newblock Disambiguating visual relations using loop constraints.
\newblock In {\em Computer Vision and Pattern Recognition (CVPR), 2010 IEEE
  Conference on}, pages 1426--1433. IEEE, 2010.

\end{thebibliography}

\end{document}